\documentclass{article}

\usepackage[a4paper, total={6in, 9in}]{geometry}

\usepackage[utf8]{inputenc} 
\usepackage[T1]{fontenc}    
\usepackage{url}            
\usepackage{booktabs}       
\usepackage{amsfonts}       
\usepackage{nicefrac}       
\usepackage{microtype}      
\usepackage{xcolor}         
\usepackage{subcaption}     

\usepackage{amsmath}
\usepackage{amssymb}
\usepackage{amsthm}
\usepackage{graphicx}
\usepackage[colorlinks=true, allcolors=blue]{hyperref}
\usepackage{cleveref}
\usepackage{booktabs}
\usepackage{multirow}

\newcommand{\half}{\frac{1}{2}}

\newcommand{\grad}{\mathrm{grad}}
\newcommand{\Id}{\mathbb{I}}
\newcommand{\ctg}{\mathrm{tcg}}

\newcommand{\getDMAffiliation}{DeepMind, London, UK} 

\newtheorem{lem}{Lemma}

\title{Implicit Riemannian Concave Potential Maps}

\makeatletter
\newcommand{\printfnsymbol}[1]{%
  \textsuperscript{\@fnsymbol{#1}}%
}
\makeatother

\author{%
  Danilo~J.~Rezende\thanks{Equal contribution} \\
  \getDMAffiliation \\
  \texttt{danilor@deepmind.com} \\
  \and 
  Sébastien~Racanière\printfnsymbol{1} \\ 
  \getDMAffiliation \\
  \texttt{sracaniere@deepmind.com} \\
}

\begin{document}

\maketitle

\begin{abstract}
  We are interested in the challenging problem of modelling densities on Riemannian manifolds with a known symmetry group using normalising flows. This has many potential applications in physical sciences such as molecular dynamics and quantum simulations. In this work we combine ideas from implicit neural layers and optimal transport theory to propose a generalisation of existing work on exponential map flows, {\it Implicit Riemannian Concave Potential Maps}, IRCPMs. IRCPMs have some nice properties such as simplicity of incorporating symmetries and are less expensive than ODE-flows. We provide an initial theoretical analysis of its properties and layout sufficient conditions for stable optimisation. Finally, we illustrate the properties of IRCPMs with density estimation experiments on tori and spheres.
\end{abstract}

\section{Introduction}

This work focuses on applications of generative models where we want to approximate a $G$-invariant target density $p: \mathcal{M} \rightarrow \mathbb{R}^+$,  $p(x) \propto e^{-u(x)}$, with support on a prescribed compact Riemannian manifold $\mathcal{M}$ and a known isometry group $G$.
This setting is important to many applications of ML to physical systems such as molecular systems \cite{kohler2020equivariant, wirnsberger2020targeted, noe2019boltzmann} and quantum field theory simulations \cite{kanwar2020equivariant, boyda2021sampling}.

To achieve this, we aim to build generative models with support on $\mathcal{M}$ whose density is $G$-invariant by construction (as opposed to learning an approximate $G$-invariance). We focus on a family of models termed normalising flows. These are generative models whose density $q_{\theta} = f_{\theta \#} \pi$ is defined via the push-forward of a base density $\pi: \mathcal{M} \rightarrow \mathbb{R}^+$ through a parametric diffeomorphism $f_{\theta}: \mathcal{M} \rightarrow \mathcal{M}$ with parameters $\theta \in \mathbb{R}^n$ (usually parametrised by a neural network), \cite{papamakarios2019normalizing}. Normalising flows can be used to construct $G$-invariant densities when the base density $\pi$ is $G$-invariant (e.g. the Haar measure) and the diffeomorphism $f$ is $G$-equivariant. This mechanism has been used to build invariant models on Euclidean manifolds \cite{rezende2019equivariant, papamakarios2019normalizing, kohler2020equivariant}, Lie groups (U$(N)^d$, SU$(N)^d$) \cite{kanwar2020equivariant, boyda2021sampling} and more general manifolds \cite{katsman2021equivariant}.

We turn our attention to normalising flows built from the gradients of a scalar potential $\psi^c: \mathcal{M} \rightarrow \mathbb{R}$. In the Euclidean space, this family of flows (referred to as {\it convex potential flows}, CPFs), where a diffeomorphism $f: \mathcal{M} \rightarrow \mathcal{M}$ is defined via the gradients of a convex potential, $f_{\theta}(x) = \nabla_x \psi_{\theta}^c(x)$\footnote{Convexity of the smooth function $\psi^c$ on a convex support is a sufficient condition for the map $f$ to be a diffeomorphism on the same support.}
was explored in \cite{huang2020convex}. A promising way to extend CPFs to more general Riemannian manifolds builds on results from optimal transport (OT) theory \cite{villani2009optimal, mccann2001polar} using cost-convex/concave potentials, which are a natural generalisation of convex/concave potentials in the Euclidean space. This has been explored in \cite{sei2013jacobian, cohen2021riemannian, rezende2019equivariant}.

A cost-concave (c-concave) potential with cost $c: \mathcal{M} \times \mathcal{M} \rightarrow \mathbb{R}$ is a function $\psi^c: \mathcal{M} \rightarrow \mathbb{R}$ that is not identically $-\infty$ for which there exists a dual function $\psi: \mathcal{M} \rightarrow \mathbb{R}$ such that
\begin{align}
\psi^c(x) &= \inf_y \half c(x, y) + \psi(y), \label{eq.def.cconcave}
\end{align} 
\cite[Definition~5.7]{villani2009optimal}, \cite{mccann2001polar}.

As observed in \cite{cohen2021riemannian} when $\mathcal{M}$ is a smooth Riemannian manifold,
OT theory has a powerful result 
that given two measures $\mu$, $\nu$ on $\mathcal{M}$, with $\mu$ being absolutely continuous with respect to the volume-measure on $\mathcal{M}$ and a cost $C[f] = \int d\mu(x) c(x, f(x))$ with $c(x,y) = d(x, y)^2$, where $d$ is the geodesic distance, there is a unique $\half d(x, y)^2$-concave potential $\psi^c$ that minimises $C[f]$ such that $\nu$ is the push-forward of $\mu$ through the diffeomorphism $f(x) = \exp_x - \nabla_x \psi^c(x)$, where $\exp_x: T_x\mathcal{M} \rightarrow \mathcal{M}$ is the Riemannian exponential map at the point $x \in \mathcal{M}$ and $\nabla_x$ is the covariant derivative at $x$, \cite[Theorem~9]{mccann2001polar} and \cite[Corollary~10.44]{villani2009optimal}. 
This result gives us the confidence that, if we can parametrise a sufficiently expressive family of c-concave potentials, it is possible to approximate an arbitrary target measure $\nu$ starting from a simple base measure $\mu$. In what follows we will only consider the case $c(x,y) = d(x, y)^2$. 

In this paper we explore normalising flows constructed from the exponential map of the gradients of cost-concave potentials as outlined above, building on work from \cite{sei2013jacobian, cohen2021riemannian, rezende2019equivariant}. Our contributions are as follows:
\begin{enumerate}
    \item We introduce {\it Implicit Riemannian Concave Potential Maps, IRCPMs} which extend {\it Riemannian Convex Potential Maps}, RCPMs from \cite{cohen2021riemannian} by allowing for much more general c-concave functions, \Cref{sec.cconcave.nnets}, and by making it easy to incorporate symmetries of the target density, \Cref{sec.symm};
    \item We extend results from implicit layers, \cite{zhang2020implicitly, amos2017optnet}, allowing higher-order implicit derivatives to be efficiently computed, \Cref{sec.higher_order_grads};
    \item We discuss conditions under which the infimum problem necessary to define the potential $\psi^c$ can be simply solved by gradient descent, Section~\ref{sec.bouding_derivatives_psi}. We give a concrete criterion for the $S^n$ case, Section~\ref{sec.case_of_Sn}, and experimentally verify that criterion on $S^2$, Section~\ref{sec.multi_modal_spheres}.
\end{enumerate}
Finally we provide proof-of-concept experimental results on toy densities on tori in \Cref{sec.multi_modal_tori} and spheres in \Cref{sec.multi_modal_spheres}.

\begin{figure}[t!]
\centering
\begin{subfigure}{.33\textwidth}
  \centering
  \includegraphics[width=0.9\textwidth]{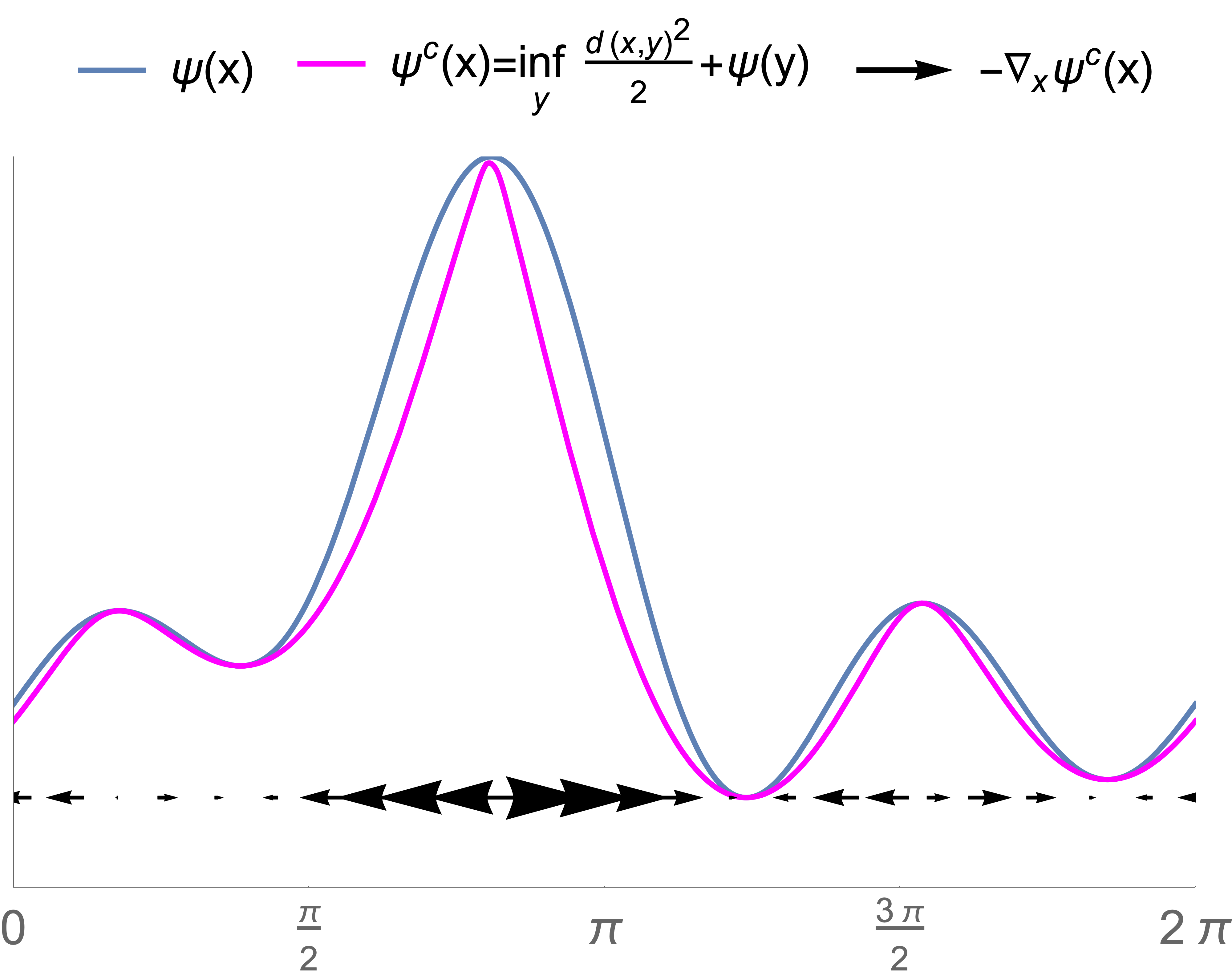}
\end{subfigure}%
\begin{subfigure}{.33\textwidth}
  \centering
  \includegraphics[width=0.9\textwidth]{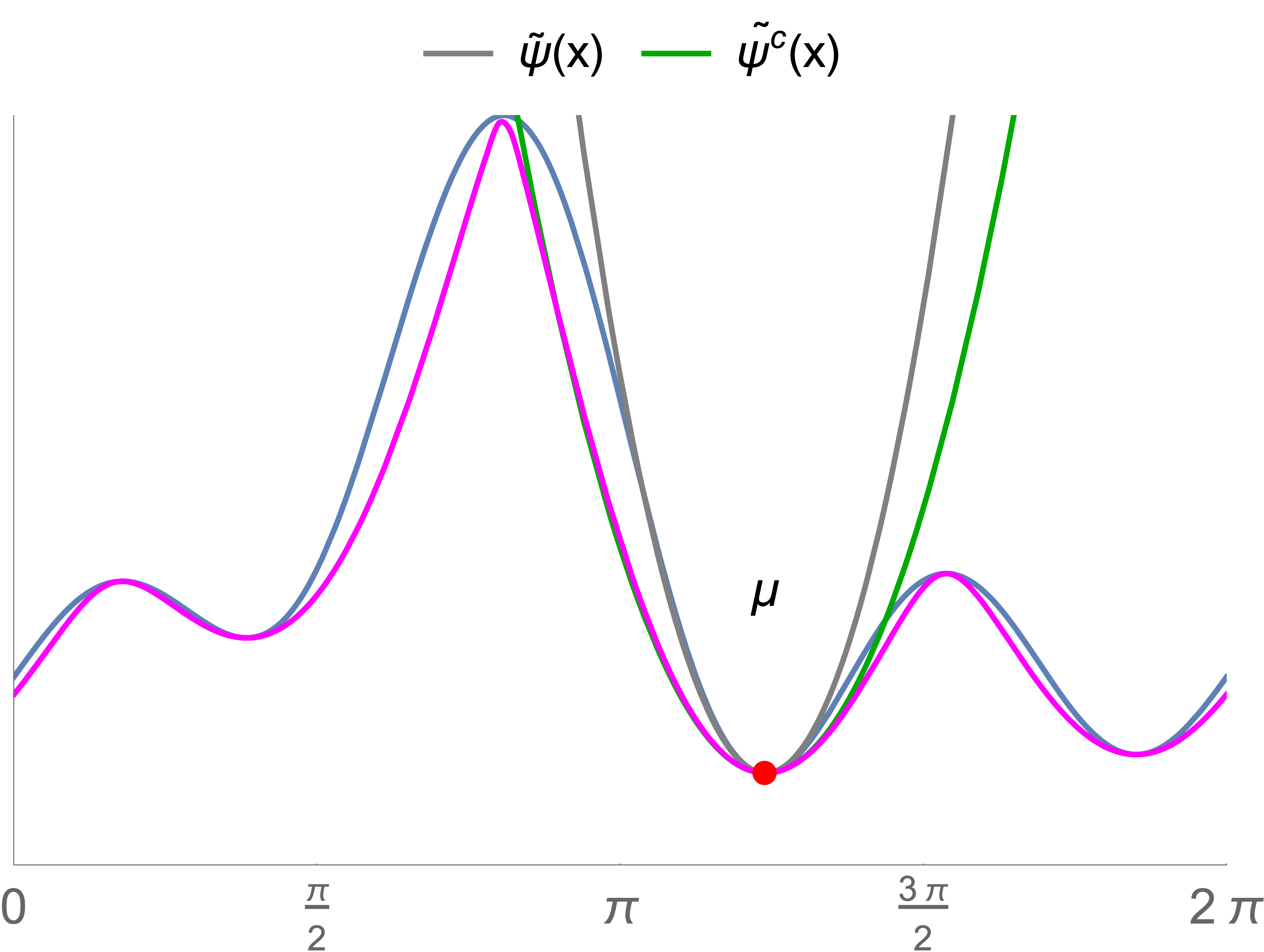}
\end{subfigure}%
\begin{subfigure}{.33\textwidth}
  \centering
  \includegraphics[width=0.9\textwidth]{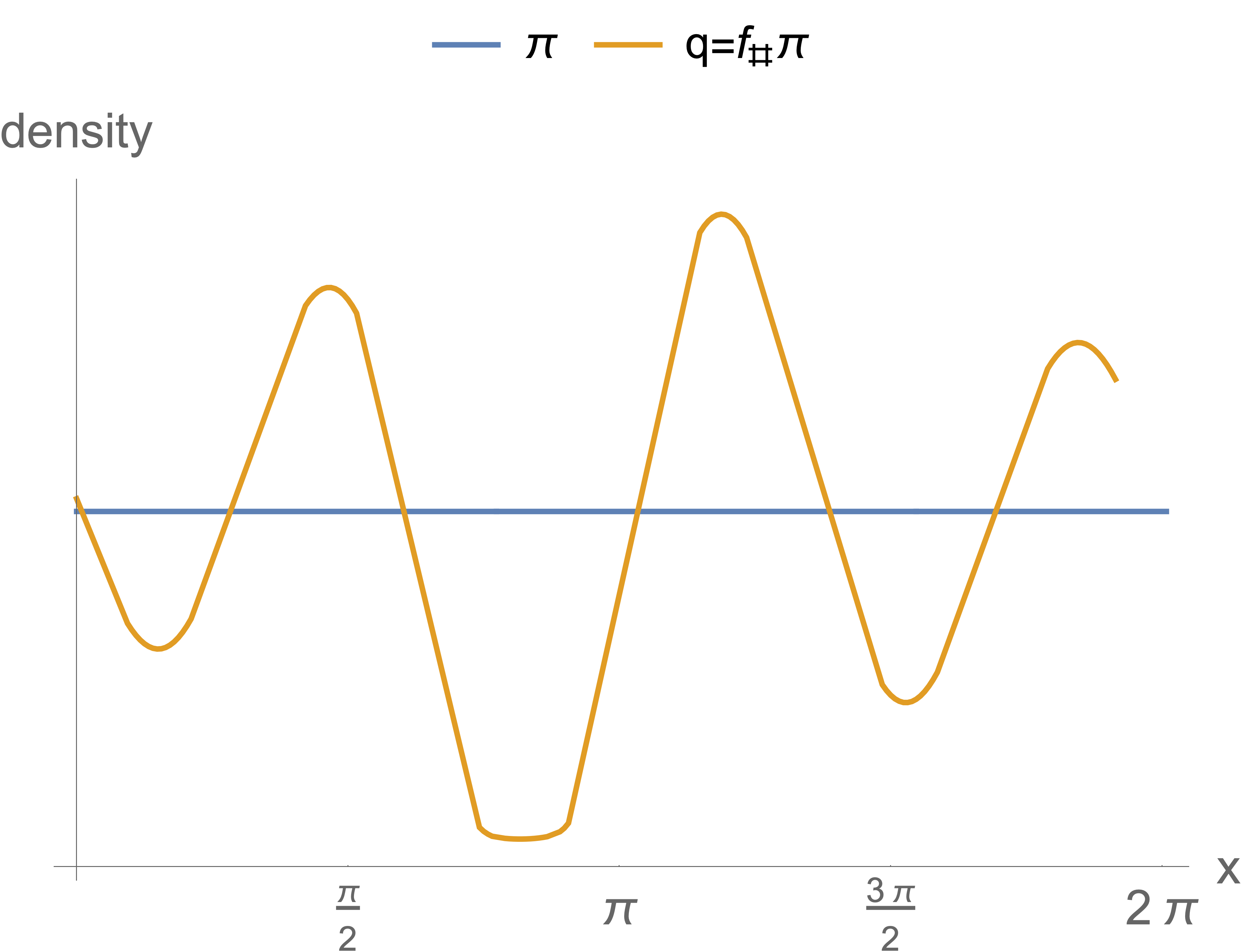}
\end{subfigure}%
\caption{Illustration of Implicit Riemannian Concave Potential Maps, IRCPMs, on the torus $\mathcal{M} = \mathbb{T}^1$. {\bf Left} The potential $\psi(x)$ is parametrised by a neural network with scalar output (blue). Its cost-concave dual $\psi^c(x)$ is defined implicitly via a minimisation problem (magenta). By construction,  $\psi^c(x) \leq \psi(x)$.
The flow is then built from the gradients of $\psi^c(x)$, $f(x) = \exp_x -\nabla_x \psi^c(x)$. In the torus with angular coordinates,  $\exp_x v = (x + v) \mod 2\pi$ so in this example $f(x) = x -\nabla_x \psi^c(x) \mod 2\pi$.
{\bf Middle} The potential $\psi$ and its cost-concave dual $\psi^c(x)$ have the same local minima, but with different second-order structure (e.g. Hessians) as indicated by the quadratic expansions $\tilde{\psi}$ and $\tilde{\psi^c}$ of $\psi$ and $\psi^c$ resp. around one local minimum. Consequently, $\psi$ and $\psi^c$ will have a similar mode structure as the target energy $u$.
{\bf Right} The density defined by the flow $q = f_{\#}\pi$ (yellow) as the push-forward of the Haar measure $\pi(x) = \frac{1}{2\pi}$ (blue). Note how the local minima in $\psi^c$ induce modes in $q$ at slightly different locations.}\label{fig.illustration}
\end{figure}

\section{Related Work}

Previous work \cite{sei2013jacobian, rezende2020normalizing} has focused on directly parametrising the c-concave function $\psi^c$ with a neural network or other parametric family. This has the advantage of allowing a closed form flow, but it is hard to implement general and expressive solutions. Even for simple manifolds such as spheres, this proves to be a challenge \cite{rezende2020normalizing, cohen2021riemannian}.

{\bf Discrete c-concave potential flows} are a manifold-agnostic solution to parametrise c-concave functions $\psi^c$ whose dual function $\psi$ is assumed to be discrete. This reduces the continuous minimisation in \Cref{eq.def.cconcave} to taking the minimum over a discrete set of values, \cite{cohen2021riemannian}. 
While this solution can be shown to enjoy universal approximation properties, it still has two main issues that we wish to address: (i) Scaling issues to high-dimensions (since discrete c-concave potentials are like mixture models it is hard to scale them to large-dimensional spaces without resorting to auto-regressive architectures)\footnote{Universality of discrete c-concave potentials relies on using $\epsilon$-nets which is a set of points on $\mathcal{M}$ such that any two points in the set are at most $\epsilon$-away from each other. The number of points in the $\epsilon$-net grows as $O(\epsilon^{-d})$, where $d$ is the dimension of $\mathcal{M}$.}; (ii) Difficulty of incorporating target symmetries (there is no computationally cheap mechanism to incorporate known symmetries of the target density into discrete c-concave potentials).

We aim to address issues (i) and (ii) by introducing implicit c-concave potentials, which allow us to work with a much broader family of smooth dual functions $\psi$, while making it straightforward to incorporate target symmetries, \Cref{sec.symm}.

{\bf Implicit components} for neural networks, defined via an inner optimisation, have been explored in many different contexts \cite{baydin2018automatic}; they have been proposed as layers \cite{zhang2020implicitly, amos2017optnet}, in meta-learning \cite{lee2019meta}, generative models \cite{kim2018semi} and planning \cite{de2018end, srinivas2018universal} to name a few applications.

In \cite{zhang2020implicitly, amos2017optnet} implicit layers defined via minimisation or fixed-point equations have been studied. The first-order gradients of implicit layers are computed using the implicit function theorem for efficiency. Our formulation follows the same principles, but we extend the methodology to compute up to third-order gradients efficiently as necessary for our application, \Cref{sec.higher_order_grads}.

{\bf Equivariant manifold flows} \cite{katsman2021equivariant} are a  general solution to build $G$-equivariant ODE flows on manifolds. Working with ODEs in manifolds introduces several technical challenges such as complex integration procedures and computationally expensive estimation of the model's likelihood.

{\bf OptimalTransport-Flows} \cite{onken2021ot} are ODE-flows defined in the Euclidean space with additional transport costs added during optimisation.

\section{Implicit Riemannian Concave Potential Maps, IRCPMs}

\subsection{Implicit c-concave functions with NNets}\label{sec.cconcave.nnets}

Inspired by the discrete cost-concave potentials proposed in \cite{cohen2021riemannian} and work on implicit neural layers (i.e. layers whose output is defined via an inner optimisation) \cite{zhang2020implicitly, amos2017optnet}, we want to explore more general c-concave functions $\psi^c$ on a manifold $\mathcal{M}$ by working directly with its definition, \Cref{eq.def.cconcave}, as a minimisation problem. We use neural networks to parametrise a family of functions $\psi_{\theta}: \mathcal{M} \rightarrow \mathbb{R}$, and implicitly define the function $\psi^c_{\theta}(x)$ via the minimisation 
\begin{align}
    \psi^c_{\theta}(x) &= \min_y h(x, y), \label{eq.nnet.phi}
\end{align}
where $h(x, y) = \half d(x, y)^2 + \psi_{\theta}(y)$.
The implicit c-concave function $\psi^c_{\theta}(x)$ is then used to define a normalising flow $f_{\theta}: \mathcal{M} \rightarrow \mathcal{M}$ via
\begin{align}
    f_{\theta}(x) &= \exp_x -\nabla_x \psi_{\theta}^{c}(x). \label{eq.def.flow}
\end{align}
This flow is illustrated in \Cref{fig.illustration}(left), where we show a concrete example of $\psi$, $\psi^c$ and $f_{\theta}$ on the torus $\mathbb{T}^1$.
The model's density $q_{\theta}: \mathcal{M} \rightarrow \mathbb{R}^+$ is defined via the push-forward of a simple base measure $\pi: \mathcal{M} \rightarrow \mathbb{R}^+$ (such as the Haar measure if $\mathcal{M}$ is a Lie group), 
\begin{align}
\ln q_{\theta}(f_{\theta}(x)) &= \ln \pi(x) - \ln |\det E_{f_{\theta}(x)} J_{f_{\theta}}(x) E_{x}^\intercal|, \label{eq.model.likelihood}    
\end{align} 
where $J_{f_{\theta}}$ is the Jacobian matrix of $f_{\theta}$ and $E_{x}$ is a matrix that projects to the tangent space $T_x\mathcal{M}$ at $x$ (this projection is necessary when working with an embedding of $\mathcal{M}$ into $\mathbb{R}^d$ instead of a local chart). In what follows we will refer to this family of normalising flows as {\it implicit Riemannian concave potential maps}, IRCPMs.

\subsection{Incorporating symmetries} \label{sec.symm}

One of the main motivations of this work is to easily incorporate symmetries into our models. Symmetries are described by a group $G$ of isometries of the manifold $\mathcal{M}$.

We could consider a G-invariant extension of discrete c-concave potentials \cite{cohen2021riemannian} via the modification 
\begin{align}
 \psi^c(x) &= \min_i \half d(x, \Gamma(y_i))^2 + \alpha_i,
\end{align}
where $d(x, \Gamma(y_i))$ is the Riemannian distance between the point $x$ and the orbit $\Gamma(y_i)$ of the point $y_i$ under the action of the symmetry group $A_g$, $d(x, \Gamma(y_i)) = \inf_g d(x, A_g y_i)$.  Unfortunately, there is no simple closed-formula for $d(x, \Gamma(y))$ on interesting settings (e.g. quantum simulations on a lattice) which makes this impractical.

We show in \Cref{lem.equivariant} that it is relatively simple to build G-equivariant IRCMPs: it is sufficient that (i) the group action $A_g: \mathcal{M} \rightarrow \mathcal{M}$ is an isometry (i.e. leaves the intrinsic distance function invariant) and (ii) that the function $\psi_{\theta}(x)$ is $G$-invariant.

\begin{lem} \label{lem.equivariant}
Let $(\mathcal{M}, h)$ be a Riemannian manifold with distance $d$ and isometry group $G$; and let $\psi : \mathcal{M} \rightarrow \mathbb{R}$ be a G-invariant scalar function. Then the diffeomorphism $f: \mathcal{M} \rightarrow \mathcal{M}$ defined via $f(x) = \exp_x -\nabla \psi^c(x)$ is G-equivariant.
\end{lem}
\begin{proof}

First we show that the cost-concave dual function $\psi^c(x)$ is G-invariant
\begin{align*}
    \psi^c(A_g x) &= \inf_y \half d(A_g x, y)^2 + \psi(y)\\
    &= \inf_y \half d( x, A_{g^{-1}} y)^2 + \psi(y) && \text{$A_g$ is an isometry with adjoint $A_{g^{-1}}$}\\
    &= \inf_y \half d( x, A_{g^{-1}} y)^2 + \psi(A_{g^{-1}} y) && \text{$\psi$ is invariant}\\
    &= \inf_z \half d(x, z)^2 + \psi(z) && \text{$A_g$ is invertible}\\
    &= \psi^c(x).
\end{align*}
It remains to show that $\exp_x -\nabla \psi^c(x)$ is $G$-equivariant if $\psi^c(x)$ is $G$-invariant. As observed in \cite{katsman2021equivariant}(Th1) the gradient of a G-invariant scalar function is G-equivariant. Therefore the term $-\nabla \psi^c(x)$ is G-equivariant.
Finally, the exponential map $\exp_x$ is equivariant to isometries. Indeed, if $x\in\mathcal{M}$, $\zeta\in T_x\mathcal{M}$ and $g\in G$, the two curves $\gamma(t)=\exp_{A_g \cdot x} A_g \cdot t\zeta$ and $\delta(t)=A_g\cdot\exp_x t\zeta$ are two geodesics with $\gamma(0)=\delta(0)=A_g \cdot x$ and $\frac{\partial}{\partial t}|_{t=0}\gamma(t) = \frac{\partial}{\partial t}|_{t=0}\delta(t) = A_g\cdot\zeta$. These two geodesics are therefore equal. Evaluating them at $t=1$ gives us the equivariance of $\exp$.
\end{proof}

\subsection{Relation between \texorpdfstring{$\psi$}{psi}, \texorpdfstring{$\psi^c$}{psic}} \label{sec.modes.psi}

An interesting observation is that local minima $\mu_i$ of $\psi$ are also local minima of $\psi^c$, as formalised in \Cref{lem.quadratic_in_quadratic_out} in the Euclidean space. Additionally, \Cref{lem.quadratic_in_quadratic_out} relates the Hessian, $M$, of $\psi$ and $M^c = \mathbb{I} + M^{-1}$ of $\psi^c$ at the local minima.
Although this observation is limited to Euclidean geometry it suggests that for points $x \in \mathcal{M}$ near the local minima $\mu_i$ of $\psi$ the flow $f (x) = \exp_x - \nabla \psi^c (x)$ will increase the 
concentration of the base density in the regions $x \approx
\mu_i$. Consequently, $\psi$ and $\psi^c$ should have a similar local minima structure as the target energy. This suggests a heuristic to initialise $\psi$ as a function of the form $\psi(x) = g(u(x)) + \Delta(x)$, where $u(x)$ is the target energy, $g(x)$ is a learned non-decreasing function and $\Delta(x)$ is a learned "correction". This relationship between $\psi$ and $\psi^c$ is illustrated in \Cref{fig.illustration}(right) using an example $\psi$ and $\psi^c$ defined on the torus $\mathbb{T}^1$.

We can empirically observe these relationships in \Cref{fig.torus.modes} and \Cref{fig.sphere_psi}
where the learned functions $\psi$ and $\psi^c$, develop the same modes as the
target density.

\begin{lem} \label{lem.quadratic_in_quadratic_out}
  Let $\mathcal{M} =\mathbb{R}^d$ and $M$ be a positive definite matrix and $\mu \in \mathcal{M}$. Then
  the c-concave function $\psi^c$ dual to $\psi = \half
  (y - \mu)^\intercal M (y - \mu)$ is given by
  \begin{align*}
      \psi^c (x) &= \half (x - \mu)^\intercal [\mathbb{I} - (\mathbb{I}+ M)^{- 1}] (x
     - \mu) + \text{cst}.
  \end{align*}
\end{lem}
\begin{proof}
  Since $M$ is assumed positive definite, $h(x, y) = \half d(x, y)^2 + \psi(y)$ has a global minimum $y^{\star}(x) = (\mathbb{I}+ M)^{-1} (x + M \mu)$ which is given by solving $\nabla_y h (x, y) = 0$.
  We can re-arrange the terms of $h(x, y)$ to a more convenient form by performing a Taylor expansion around $y=0$ up to second-order (since $h$ is quadratic, the expansion is exact),
  \begin{align*}
    h(x, y) &= \half \|x\|^2 - (x + M \mu)^\intercal y + \half y^\intercal (\mathbb{I}+ M)y + \half\mu^\intercal M\mu.
  \end{align*}
  Replacing $y^{\star}(x)$ above gives
  \begin{align*}
    \psi^c (x) & = h (x, y^{\star} (x)) = \half \| x \|^2 - \half (x
    + M \mu)^\intercal (\mathbb{I}+ M)^{- 1} (x + M \mu) + \half\mu^\intercal M\mu.
  \end{align*}
  Now we note that in the quadratic form above $\partial_x h (x, y^{\star} (x)) =0 \Rightarrow x = \mu$, so it is a quadratic form centred at $\mu$ with Hessian matrix $H=\mathbb{I} - (\mathbb{I}+ M)^{- 1}$, so it can be re-written as 
  \begin{align*}
   \psi^c (x) &= \text{cst} + \half (x - \mu)^\intercal H (x -\mu). & 
  \end{align*}
\end{proof}

\section{Solving the infimum problem}

Given a generic $\psi$, computing $\psi^c(x) = \inf_y h(x, y)$ can be a difficult problem. We are interested in finding conditions on $\psi$ which guarantee this problem can be solved with gradient descent.

\subsection{Bounding the derivatives of \texorpdfstring{$\psi$}{psi}}\label{sec.bouding_derivatives_psi}

Let $\varphi(x)$ be the function obtained by following gradient descent along $y$ on $h(x, y)$, initialised at $y=x$. Computing $\varphi$ is a much simpler problem than $\psi^c$. Below we will see that by controlling the norm of the gradient of $\psi$, as well as the operator norm of its Hessian, we can guarantee that $\psi^c=\varphi$ everywhere.

First, let's see how to ensure that the global minimum is in a ball around $x$.

\begin{lem}\label{lem:gradient-psi}
Assume $\lambda > 0$ is such that $||\grad(\psi)|| < \lambda/2$ everywhere, then if $d(x, y) > \lambda$ we have $h(x, y) > h(x, x)$. In particular, the infimum in $y$ of $h(x, y)$ is within a ball of radius $\lambda$ around $x$.
\end{lem}
\begin{proof}
Let $y$ be such that $d(x, y) > \lambda$. Since $\mathcal{M}$ is complete, there exists $\zeta\in T_x\mathcal{M}$ with $y=\exp_x(\zeta)$ and $d(x, y) = ||\zeta||$. For $t\in [0, 1]$, let $\alpha(t)=\psi(\exp_x(t\zeta))$.

We have
\begin{equation}
h(x, y) - h(x, x) = \half||\zeta||^2 + \alpha(1) - \alpha(0) = \half||\zeta||^2 + \int_0^1 dt~\alpha'(t)
\end{equation}

We have $|\int_0^1 dt~\alpha'(t)| \leq \int_0^1 dt~ |\alpha'(t)| \leq \int_0^1 dt~ ||\grad(\psi)||~||\zeta|| \leq \half\lambda ||\zeta||$. This means $h(x, y) - h(x, x) \geq \half||\zeta||^2 - \half\lambda||\zeta||$. Since the function $\half \mathbb{T}^2-\half\lambda t$ is positive for $t > \lambda$, the result follows.

\end{proof}

The above Lemma~\ref{lem:gradient-psi} lets us constrain the location of the global minimum inside a ball. Let's now see a condition to ensure that this minimum is unique inside this ball, and is reached by gradient descent. We write $H_\psi$ for the Hessian of $\psi$, and $||H_\psi||$ for the operator norm of $H_\psi$.

\begin{lem}\label{lem:phi-is-varphi}
We can choose $\lambda>0$ and $\eta>0$ such that if $||\grad(\psi)|| < \lambda/2$ and $||H_\psi|| < \eta$ everywhere, then for any $x$, $h(x, y)$ has its global minimum within the ball $B_x(\lambda)$. Gradient descent along $y$ on $h(x, y)$, initialised at $y=x$, finds this minimum. In other words, $\psi^c=\varphi$.
\end{lem}
\begin{proof}
Assume that $\lambda$ is small enough such that for any $x$ in $\mathcal{M}$, the map $\exp_x$ is a diffeomorphism from the euclidean unit closed ball in $T_x\mathcal{M}$ onto its image $\bar{B}_x(\lambda)$. Also, as in Lemma~\ref{lem:gradient-psi}, assume that $||\grad(\psi)|| < \lambda/2$ everywhere.

Since $h(x, y)$ is a sum of two terms, $\half d(x, y)^2$ and $\psi(y)$, its Hessian $H_h$ with respect to $y$ is also a sum of two terms, respectively $H_c$ and $H_\psi$. At $y=x$, the Hessian $H_c$ is exactly the identity $\Id$. This means that for any $\eta$ in $(0, 1)$, if $\lambda$ is small enough, then $||H_c - \Id|| < 1-\eta$ in $\bar{B}_x(\lambda)$. Now, assume that $||H_\psi|| < \eta$ everywhere. This implies that $||H_c - \Id + H_\psi|| < 1$ and $H_h = \Id + (H_c - \Id + H_\psi)$ is invertible with eigenvalues in $(0, 2)$. In particular it is positive definite at any $y$ in $\bar{B}_x(\lambda)$. We can immediately conclude from \cite[Lemma~2.2]{milnor2016morse} that critical points of $h(x, \cdot)$ are all isolated minima inside $B_x(\lambda)$.
The vector field $-\grad_y h$ along the boundary of $\bar{B}_x(\lambda)$ points strictly inward the ball. By \cite[Cor.~8.3.1]{jost2008riemannian}, the ball is the disjoint union of its stable manifolds $W^s(q)$\footnote{A stable manifold $W_s(q)$ of a critical point $q$ is the set of points on the manifold for which the integral curves $\gamma(t)$ of $-\grad_h h$ converge to $q$ at $t$ goes to $+\infty$. The dimension of this manifold is given by the number of positive eigenvalues of the Hessian at the critical point $q$. See \cite{jost2008riemannian} for more details.}, where $q$ runs over all the critical points of $h$. Since at every critical points of $h$ the Hessian is positive definite, all the submanifolds $W^s(q)$ are open and of the same dimension as $\mathcal{M}$. Since $\mathcal{M}$ is connected, there is only one critical point $q$ and $W^s(q) = \bar{B}_x(\lambda)$. By definition of $W^s(q)$, gradient descent along $-\grad h$ leads to $q$.
\end{proof}

The above Lemma~\ref{lem:phi-is-varphi} does not tell us how to choose $\lambda$ and $\eta$. This choice will depend on the curvature of the manifold, as it controls the eigenvalues of the Hessian $H_y$ \cite[Section 2.2]{ferreira2006hessian}.

\subsection{Implicit higher-order gradients}\label{sec.higher_order_grads}

From \Cref{lem:phi-is-varphi}, under Lipschitz constraints on $\psi$ and its gradients, the global minimum in \Cref{eq.def.cconcave} will coincide with the local minima $\partial_{y} h(x, y) = 0$. This allows us to formulate \Cref{eq.nnet.phi} as an equality constraint implicit layer \cite{zhang2020implicitly, amos2017optnet}, $F^0_i(x, y) = 0$,  where 
\begin{align}
F^0_i(x, y) &= \partial_{y_i} h(x, y)=0 \quad \forall x \in \mathcal{M}.
\end{align}
The first-order gradients of the minimiser $y^{\star}(x)$ with respect to $x$ can be efficiently computed via the implicit function theorem. 

IRCPMs are built using the first-order gradients of $\psi^c(x)$, which require first-order gradients of $y^{\star}(x)$ with respect to $x$. As a consequence, the log-det-Jacobian terms in \Cref{eq.model.likelihood} will involve second-order derivatives of $y^{\star}(x)$ with respect to $x$. 
Finally, optimisation of the model's likelihood will require computing an extra derivative with respect to the model's paramters $\theta$. This results in 
third-order derivatives of $y^{\star}(x)$ (second-order in $x$ and an extra derivative with respect to $\theta$).

We can efficiently compute higher-order derivatives of $y^{\star}(x)$ with respect to $x$ and $\theta$ by considering a recursive computation of vector-Jacobian products $u_i = v^\intercal \partial_x y_i^{\star}$. As in \cite{zhang2020implicitly, amos2017optnet}, $u_i$ are obtained by solving a linear system of equations
\begin{align}
    F^1_i(x,u) := v^\intercal \nabla_x F^0_i(x, y) &= v^\intercal \partial_x F^0_i(x, y) + u^\intercal \partial_y F^0_i(x, y) = 0. \label{eq.vjp1}
\end{align}
Second and higher-order vector-Jacobian products can be obtained by a recursive application of \Cref{eq.vjp1}. More precisely, let the $a$-th order vector-Jacobian product be $u^a$. If $u^a$ is implicitly defined via the (linear) equality constraint $F_i^a (x, u^a) = 0$ then, the $(a+1)$-th order vector-Jacobian product $u^{a+1} = v^\intercal_{a+1} \partial_x u^a$ is given by $F_i^{a+1} (x, u^{a+1}) = v_{a + 1}^\intercal \partial_x F_i^a (x, u) + [u^{a+1}]^\intercal \partial_u
F_i^a (x, u) = 0$.

\subsection{The case of \texorpdfstring{$S^n$}{Sn}}\label{sec.case_of_Sn}

The Hessian of the squared distance, as well as its eigenvalues, is computed in \cite{ferreira2006hessian, pennec2017hessian}. Using this, we will offer practical choices for $\lambda$ and $\eta$.

On the unit sphere $S^n$, according to \cite[Eq.~$3$]{ferreira2006hessian} or \cite[Section~$3,1$]{pennec2017hessian}, the Hessian $H_c$ of $\half d(x, y)^2$ with respect to $y$ has one eigenvalue $1$ and all other eigenvalues equal to $\ctg(t)$, where $t$ is the distance between $x$ and $y$, and $\ctg(t) = t\cot{t}$ is a strictly decreasing function on $[0, \pi / 2]$ with $\ctg(0) = 1$ and $\ctg(\pi/2) = 0$.

Assume a $\lambda$ has been chosen. What is the largest value of $\eta$ we can choose? From the proof of Lemma~\ref{lem:phi-is-varphi}, we need that $||H_c - \Id|| < 1 - \eta$. On $S^n$, we know from the eigenvalues of $H_c$ that $||H_c - \Id|| = 1 - \ctg(t)$ if $||y|| = t$. So, inside $B_x(\lambda)$, the maximum value of $||H_c - \Id||$ is $1 - \ctg(\lambda)$ and we can choose $\eta$ at most $\ctg(\lambda)$. We will use this result in our experiment in Section~\ref{sec.multi_modal_spheres}.

\section{Optimisation and Experiments}

In each experiment we trained the model to approximate a given target density $p(x) \propto e^{-u(x)}$ for some energy function $u(x)$ defined either on a torus $\mathbb{T}^2$ or on the sphere $S^2$. The model used for the experiments is either an IRCPM or a stack of IRCPMs as specified in each experiment. 

The function $\psi_{\theta}(x)$ is parametrised with a two-layer MLP with softplus non-linearities with layer sizes [32, 32, 1] for the experiments on the torus and layer sizes [128, 128, 1] for the experiments on the sphere. The parameters $\theta$ of $\psi_{\theta}(x)$ are optimised using the Adam optimiser \cite{kingma2014adam} to minimise the KL-divergence between the density defined by the model $q_{\theta}$ and the target $p$,
\begin{align}
   \text{KL}(q_{\theta}; p) &= \mathbb{E}_{\pi(x)}[\ln \pi(x) - \ln |\det E_{f_{\theta}(x)} J_{f_{\theta}}(x) E_{x}^\intercal| + u(f_{\theta}(x))] + \ln Z,
\end{align}
where $Z = \int_{\mathcal{M}} d\mu(x) e^{-u(x)}$ is the normaliser of the target density.
The gradients of the term $\ln |\det E_{f_{\theta}(x)} J_{f_{\theta}}(x) E_{x}^\intercal|$ with respect to $\theta$ are estimated using the stochastic estimator from \cite{dong2017scalable, huang2020convex} unless stated otherwise. For all experiments we report estimated values of $\text{KL}(q_{\theta}; p)$ and the effective sample size (ESS), \cite{doucet2001introduction, liu2001monte}, using $N=20000$ samples from the model.

\subsection{Targets on \texorpdfstring{$\mathbb{T}^2$}{T2}} \label{sec.multi_modal_tori}

We tested IRCPMs on 3 target densities on the torus $\mathcal{M} = \mathbb{T}^2$. Their expressions are detailed in \Cref{table.psi}. These targets test the ability of our model to capture modes as well as to incorporate symmetries. For the experiment with symmetries the target has a continuous symmetry group $(x_1, x_2) \rightarrow (x_1 + c, x_2 + c)$ for some arbitrary angle $c$. The symmetry was exactly incorporated into the model by using invariant features $h = x_1 - x_2$ as input to the function $\psi_{\theta}$, this architecture is referred to as "Symmetric MLP" in \Cref{table.psi}.
The results are shown in \Cref{fig.torus.modes} and \Cref{table.psi}. In all cases the model reaches ESS~$\ge 94\%$. We observed, as suggested by the analysis in \Cref{sec.modes.psi}, that the learned $\psi$ function has similar mode structure as the target energy. 

In \Cref{table.psi} we compared a free form $\psi$ (parametrised by a neural network) with different parametrisations inspired by \Cref{sec.modes.psi} which use the target energy $u(x)$ to construct $\psi$ in a way that leverages its local structure. The experiments suggest that combining learned components with the target energy while preserving its mode structure allows the model to reach competitive ESS while using a {\it much smaller} number of parameters, which is a very desirable property.

\begin{table}[t!]
\centering
\begin{tabular}{@{}cccc@{}}
\toprule
\textbf{Target}                              & \textbf{$\psi$ architecture}                                          & \textbf{\# parameters}             & \textbf{ESS}                  \\ \midrule
\multicolumn{1}{|c|}{\multirow{4}{*}{$u_1(x) = \sin(\pi (x_1 - x_2))$}} & \multicolumn{1}{c|}{MLP}                                & \multicolumn{1}{c|}{1249} & \multicolumn{1}{c|}{$84.2\%$} \\ \cmidrule(l){2-4} 
\multicolumn{1}{|c|}{}                       & \multicolumn{1}{c|}{Symmetric MLP}                      & \multicolumn{1}{c|}{1249} & \multicolumn{1}{c|}{$94.4\%$} \\ \cmidrule(l){2-4} 
\multicolumn{1}{|c|}{}                       & \multicolumn{1}{c|}{$\alpha u$}                         & \multicolumn{1}{c|}{1}    & \multicolumn{1}{c|}{$97.8\%$} \\ \cmidrule(l){2-4} 
\multicolumn{1}{|c|}{}                       & \multicolumn{1}{c|}{$\alpha u + \nabla u^\intercal M \nabla u$} & \multicolumn{1}{c|}{4}    & \multicolumn{1}{c|}{$98.6\%$} \\ \midrule
\multicolumn{1}{|c|}{\multirow{3}{*}{$u_2$ as in \cite[Table 2, multi-modal]{rezende2020normalizing}}} & \multicolumn{1}{c|}{MLP}                                & \multicolumn{1}{c|}{1249} & \multicolumn{1}{c|}{$99.9\%$} \\ \cmidrule(l){2-4} 
\multicolumn{1}{|c|}{}                       & \multicolumn{1}{c|}{$\alpha u$}                         & \multicolumn{1}{c|}{1}    & \multicolumn{1}{c|}{$95.6\%$} \\ \cmidrule(l){2-4} 
\multicolumn{1}{|c|}{}                       & \multicolumn{1}{c|}{$\alpha u + \nabla u^\intercal M \nabla u$} & \multicolumn{1}{c|}{4}    & \multicolumn{1}{c|}{$96.9\%$} \\ \midrule
\multicolumn{1}{|c|}{\multirow{3}{*}{$u_3(x) = 1.4 \sin(2 \pi x_1/3) \sin(2 \pi x_2/3)$}} & \multicolumn{1}{c|}{MLP}                                & \multicolumn{1}{c|}{1249} & \multicolumn{1}{c|}{$98.2\%$} \\ \cmidrule(l){2-4} 
\multicolumn{1}{|c|}{}                       & \multicolumn{1}{c|}{$\alpha u$}                         & \multicolumn{1}{c|}{1}    & \multicolumn{1}{c|}{$98.4\%$} \\ \cmidrule(l){2-4} 
\multicolumn{1}{|c|}{}                       & \multicolumn{1}{c|}{$\alpha u + \nabla u^\intercal M \nabla u$} & \multicolumn{1}{c|}{4}    & \multicolumn{1}{c|}{$99.3\%$} \\ \bottomrule
\end{tabular}\caption{Comparison of different parametrisations of $\psi$ inspired by the analysis of \Cref{sec.modes.psi} on the torus $\mathbb{T}^2$. For targets with a continuous symmetry (top row), free-form models under perform considerably relative to models which incorporate the symmetry exactly. Architectures that incorporate the mode structure of the target can reach comparable or better ESS (higher is better) with substantially fewer parameters.}\label{table.psi}
\end{table}

\begin{figure}[t!]
\centering
\includegraphics[width=\textwidth]{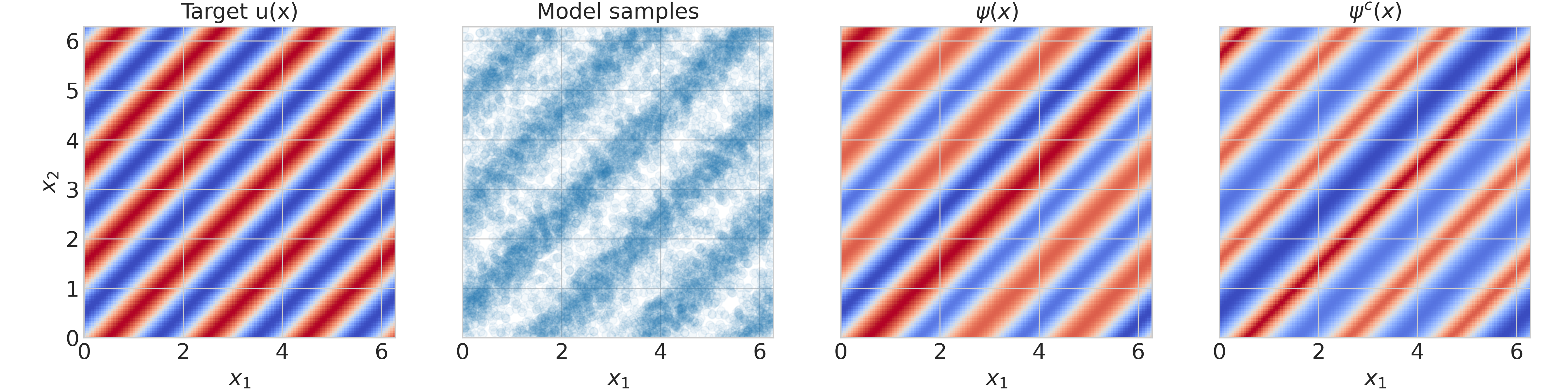}
\includegraphics[width=\textwidth]{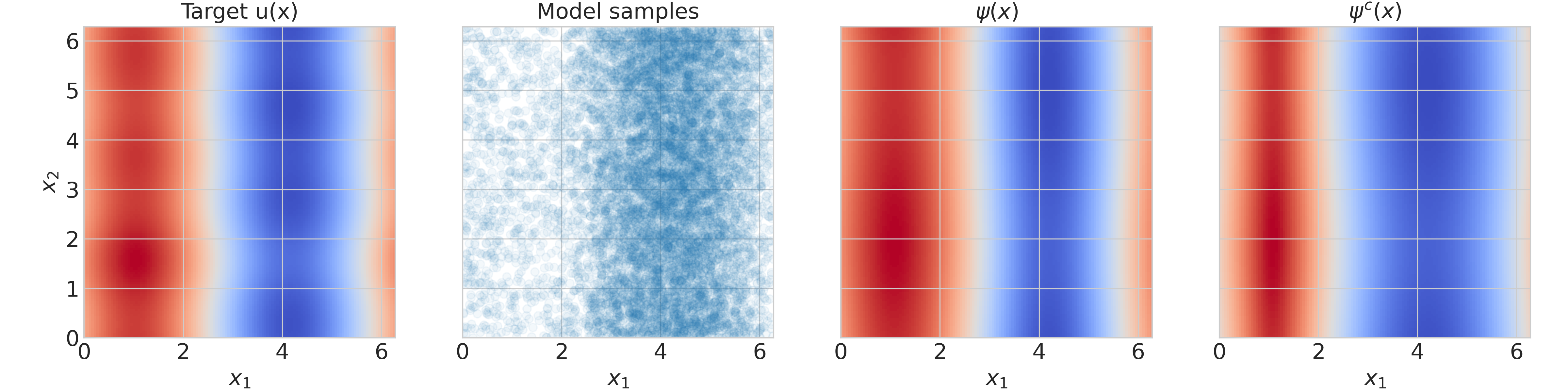}
\includegraphics[width=\textwidth]{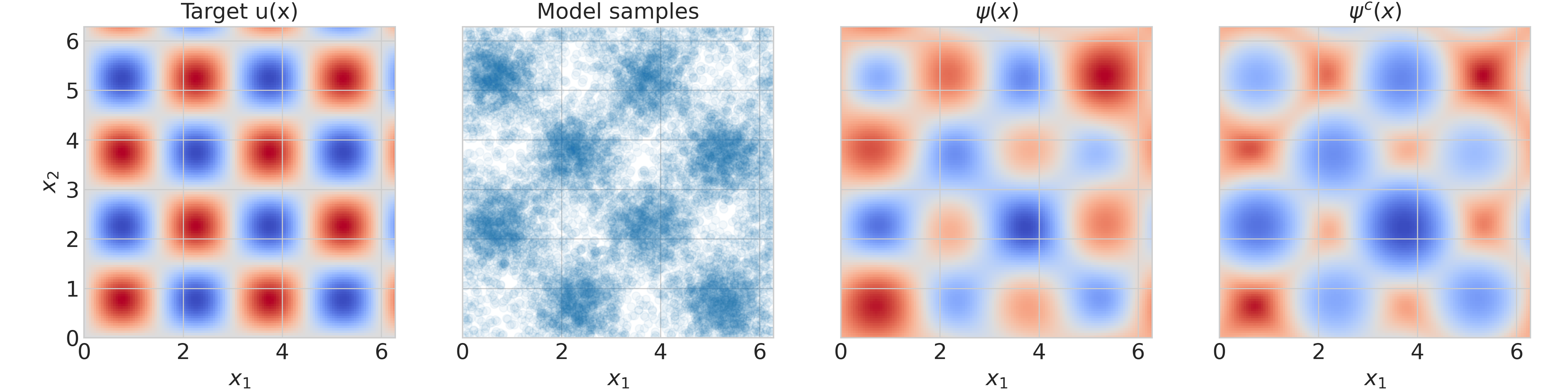}
\caption{Target densities on the torus $\mathbb{T}^2$. Columns represent, from left to right, the target energy $u(x)$, samples from the learned model, the learned function $\psi$ and its dual $\psi^c$.
Below, plot rows are counted from top to bottom. The target expressions and quantitative results are detailed in \Cref{table.psi}.
{\bf First row:} Target $u1$, multi-modal with continuous symmetry group $U(1)$.
{\bf Second row:} Target $u2$ with 3 modes.
{\bf Third row:} Target $u3$ with 8 modes.
As suggested from \Cref{sec.modes.psi}, the learned functions $\psi$ and its cost-concave dual $\psi^c$ have the same mode structure (extreme points) as the target energy $u(x)$, but with different second-order structure (Hessian) around the modes.}
\label{fig.torus.modes}
\end{figure}

\subsection{Multi-modal target on \texorpdfstring{$\mathbb{S}^2$}{S2}}
\label{sec.multi_modal_spheres}

We tested our setup on the sphere $S^2$ by choosing a target density from \cite{rezende2020normalizing}, which was also used in \cite{cohen2021riemannian}. The target is a mixture of the form
\begin{equation}
    p(x) \propto \sum_{k=1}{4}e^{10x^\intercal T_{s\rightarrow e}\mu_k}
\end{equation}
where $\mu_1=(0.7, 1.5)$, $\mu_2=(-1, 1)$, $\mu_3=(0.6, 0.5)$, $\mu_4=(-0.7, 4)$, $T_{s\rightarrow e}$ maps from spherical to Euclidean coordinates, and $x\in\mathbb{R}^3$ is a point on the embedded sphere in Euclidean coordinates.

Our model consists of a stack of three IRCPMs with shared parameters, as we found that stacking the same transformation lead to more stable training. The neural network $\psi$ is an MLP that takes the euclidean coordinates of points in $S^2\subset \mathbb{R}^3$ as inputs. The output of the MLP is divided by $20$, which is equivalent to initialising the weights of the last layer of the MLP closer to $0$. We also found this to help with training stability and final performance. Finally, we computed the logarithmic determinant of the Jacobian exactly inside the KL, and did not need stochastic estimators of the gradients.

The performance of our model is on par with the RCPM model of \cite{cohen2021riemannian}, with an ESS above $99\%$ and a KL of around $0.003$ nats. Figure~\ref{fig.sphere_psi} shows the learnt $\psi$. Note how the modes of $\psi$ are placed nearby the modes of the target density in Fig.~\ref{fig.sphere_four_modes_density}.

We checked that the bounds derived in Sections~\ref{sec.case_of_Sn} are satisfied by our trained model. For this purpose, we sampled a million points uniformly at random on $S^2$, and computed the norm of the gradient of $\psi$ as well as the operator norm of the Hessian $H_\psi$. We observed $\max_{S^2}||\grad \psi|| < 0.24$, and $\max_{S^2}||H_\psi|| < 0.84$. Since $\ctg(0.24)=0.981 > 0.84$, our model sits comfortably within the bounds derived in Section~\ref{sec.case_of_Sn} and we conclude that it is safe to use gradient descent to find $\inf_y h(x, y)$.

Given the success of simpler models $\alpha u$ and $\alpha u + \nabla u^\intercal M \nabla u$ on the torus, Table~\ref{table.psi}, we tried those same models on $S^2$. Despite the qualitative similarity between $\psi$ Fig~\ref{fig.sphere_psi} and the target energy Fig~\ref{fig.sphere_four_modes_energy} of our trained MLP model, we found the simpler models to be unstable and not competitive so more work is needed to take advantage of the similarities between $\psi$ and $u$.
Finally, we trained a model with a single or only $2$ stacked transformations instead of $3$ as above. We found the models with stack height $2$ or $3$ to be qualitatively equivalent: scatter plots, as well as ESS and KL were indistinguishable. The only difference was that with only $2$ stacked transformations, we observed $\max ||\grad \psi|| = 0.365$ and $\max ||H_\psi|| = 1.11$. These values are too high and do not satisfy the bounds of Section~\ref{sec.case_of_Sn}. This suggests that the bounds of that section might be too restrictive, and future work could improve on them. With a single transformation, the values of $\max ||\grad \psi||$ measured during training were even higher. Training of the model was unstable, with only some seeds reaching high ESS, while for other the ESS initially rose before collapsing.

\begin{figure}[t!]
\centering
\begin{subfigure}{.23\textwidth}
  \centering
  \includegraphics[width=3cm]{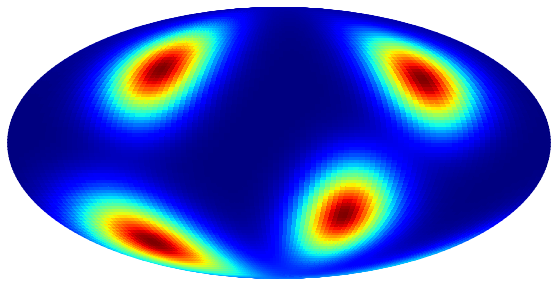}
  \caption{}
  \label{fig.sphere_four_modes_density}
\end{subfigure}%
\begin{subfigure}{.23\textwidth}
  \centering
  \includegraphics[width=3cm]{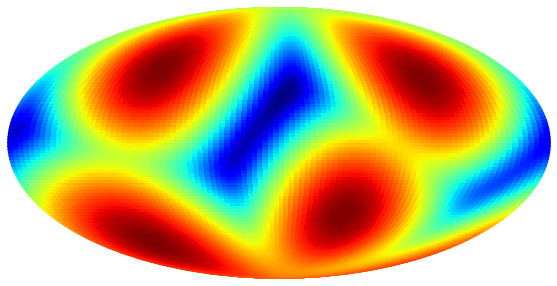}
  \caption{}
  \label{fig.sphere_four_modes_energy}
\end{subfigure}%
\begin{subfigure}{.23\textwidth}
  \centering
  \includegraphics[width=3cm]{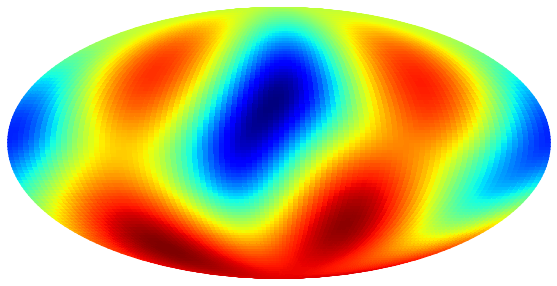}
  \caption{}
  \label{fig.sphere_psi}
\end{subfigure}%
\begin{subfigure}{.23\textwidth}
  \centering
  \includegraphics[width=3cm]{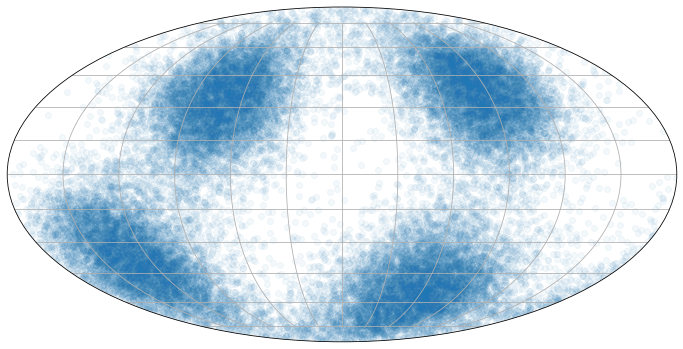}
  \caption{}
  \label{fig.sphere_model_scatter}
\end{subfigure}%
\caption{Target density \ref{fig.sphere_four_modes_density} with $4$ modes on the sphere $S^2$. The modes of $-\psi$ \ref{fig.sphere_psi} roughly match those of the target log density \ref{fig.sphere_four_modes_energy}. Samples of the learnt model ($\text{ESS}=99.5\%$, $\text{KL}=0.003$) are shown in \ref{fig.sphere_model_scatter}.}
\end{figure}

\section{Discussion}

We introduced {\it Implicit Riemannian Concave Potential Maps}, IRCPMs. A class of normalising flow models that uses implicit layers (that is, layers defined via a minimum problem) to define cost-concave potentials for exponential map flows. It extends prior work \cite{cohen2021riemannian} by considering a more general class of cost-concave functions, which allow for easy incorporation of known symmetries of the target density. We provided an initial theoretical analysis of the proprieties of these flows and lay-ed out conditions for stable training. In particular, the relationship between $\psi$ and the target energy is very intriguing. If we can better understand this relationship, perhaps we'll be able to come up with more effective parametrisation of $\psi$.

Similarly to other recent work on flows on Riemannian manifolds \cite{cohen2021riemannian, katsman2021equivariant}, the model definition is agnostic to the underlying manifold structure, but its practical implementation requires the ability to efficiently compute the exponential map, covariant derivatives and geodesic distances. 

IRCPMs have no explicit formula for the log-det-Jacobian terms. It therefore relies on stochastic estimators for scalable training \cite{dong2017scalable, huang2020convex} but evaluation of such models in high-dimensional settings based on exact likelihoods is still a challenge. This challenge is common to all flow models without explicit log-det-Jacobian \cite{papamakarios2019normalizing, grathwohl2018ffjord, cohen2021riemannian, katsman2021equivariant, mathieu2020riemannian} and not specific to our proposed model.

The presented experiments are only to provide an illustration and proof-of-concept of different aspects of IRCPMs. Since the target manifolds are low-dimensional and the neural networks used are relatively small, we used a naive backpropagation through the minimisation in \Cref{eq.nnet.phi} for optimisation. However we emphasise that a much more scalable training can be performed using the implicit higher-order gradients from \Cref{sec.higher_order_grads} (both in terms of memory and compute).

\bibliographystyle{alpha}
\bibliography{main}



\end{document}